\documentclass[12pt]{l4dc2020} 
\usepackage{color}

\title[A Finite-Sample Deviation Bound for Stable Autoregressive Processes]{A Finite-Sample Deviation Bound for Stable Autoregressive Processes}
\usepackage{times}
\author{%
 \Name{Rodrigo A. Gonz\'alez} \Email{grodrigo@kth.se}\\
 \Name{Cristian R. Rojas} \Email{crro@kth.se}\\
 \addr Division of Decision and Control Systems, KTH Royal Institute of Technology, Stockholm, Sweden
}

\begin{document}
\maketitle
\begin{abstract}%
In this paper, we study non-asymptotic deviation bounds of the least squares estimator for Gaussian AR($n$) processes. By relying on martingale concentration inequalities and a tail-bound for $\chi^2$ distributed variables, we provide a concentration bound for the sample covariance matrix of the process output. With this, we present a problem-dependent finite-time bound on the deviation probability of any fixed linear combination of the estimated parameters of the AR$(n)$ process. We discuss extensions and limitations of our approach. 
\end{abstract}
\section{Introduction}
Autoregressive (AR) processes are ubiquitous in engineering sciences, as they are applied in econometrics, time series analysis \citep{box2015time}, system identification \citep{ljung1998system}, signal processing \citep{kay1993fundamentals}, machine learning and control.

Given sampled data, the identification of the parameters of an AR process is usually done by ordinary least squares, which is known to have asymptotically optimal statistical performance \citep{mann1943statistical,durbin1960estimation} and is related to Maximum Likelihood in a Gaussian framework. Despite its success in practical applications, most analyses of the least squares method are asymptotic. Finite-time analyses of this method are still rare in the literature, despite being important for computing the number of samples needed for achieving a specified accuracy, deriving finite-time confidence sets, and designing robust control schemes. Non-asymptotic performance bounds have been historically difficult to derive since most of the classical statistical methods are better suited for asymptotic results.

In recent years, new statistical tools from the theory of self-normalizing processes \citep{pena2008self} and high dimensional probability \citep{wainwright2019high} have shown to be useful for analyzing a wide range of regression models. These tools have impulsed research on finite-time properties of the least squares estimator, with unifying efforts from the system identification, control and machine learning communities. Among topics of interest, we can find sample complexity bounds \citep{jedra2019sample}, $1-\delta$ probability bounds on parameter errors \citep{sarkar2019near}, and confidence bounds \cite[Chap. 20]{lattimore2018bandit}. 

Even though autoregression is a key aspect in dynamical systems and regression models, finite-time properties of AR($n$) processes have not yet been studied deeply. AR($n$) processes are of particular interest, as they build the foundations for studying general regression models such as ARX and ARMAX models, which are widely used in linear system identification \citep{ljung1998system}. Autoregressive processes are also essential for two-stage ARMA estimation algorithms \citep{stoica2005spectral}, and for speech production models \citep{makhoul1975linear}. For a greater understanding on how least squares performs on different autoregressive processes for finite-sample data, here we perform a non-asymptotic analysis of the least squares estimator of the coefficients of these processes. In summary, the main results of this paper are:
\begin{itemize}
\item
via martingale concentration inequalities and bounds on $\chi^2$-distribution tails, we derive a finite-time problem-dependent concentration bound for the sample data covariance matrix of an $n$-th order autoregressive process;
\item
using the previous result, we provide a bound on the deviation of any fixed linear combination of the parameters of an AR($n$) process around its true value, such that larger deviations occur only with probability at most $\delta$.
\end{itemize}

%ARXIV VERSION
The rest of this paper is organized as follows. Our work is put into context in Section \ref{sec:prior}, and we define our notation in Section \ref{sec:notation}. In Section \ref{sec:problemformulation} the problem is explicitly formulated, and a preliminary result is given. We state and prove our concentration inequalities for the sample covariance matrix and deviation bound of the parameters of an AR($n$) process in Section \ref{sec:mainresults}. Section \ref{seclemma} provides the proof for a key result used in the previous section, and a discussion of the results is presented in Section \ref{sec:discussion}. We also include an Appendix with supplementary material of proofs and auxiliary results.

%The rest of this paper is organized as follows. Our work is put into context in Section \ref{sec:prior}. We define our notation in Section \ref{sec:notation}, and explicitly formulate the problem in Section \ref{sec:problemformulation}. Preliminaries and lemmas are presented in Section \ref{sec:preliminaries}. After this, we state and prove our concentration inequalities for the sample covariance matrix and deviation bound of the parameters of an AR($n$) process in Section \ref{sec:mainresults}. Discussion of the results and conclusions are presented in Section \ref{sec:discussion}. 

%CONFERENCE VERSION
%The rest of this paper is organized as follows. Our work is put into context in Section \ref{sec:prior}, and we define our notation in Section \ref{sec:notation}. In Section \ref{sec:problemformulation} the problem is explicitly formulated, and a preliminary result is given. We state and prove our concentration inequalities for the sample covariance matrix and deviation bound of the parameters of an AR($n$) process in Section \ref{sec:mainresults}.\footnote{Supplementary material regarding proofs of Lemmas \ref{lemma2} and \ref{lemma3}, and auxiliary results, is available in the preprint version of this paper \citep{gonzalez2019deviation}.} Section \ref{seclemma} provides the proof for a key result used in the previous section, and a discussion of the results is presented in Section \ref{sec:discussion}.

\section{Relation to prior work}
\label{sec:prior}
In a time-series context, \cite{bercu1997large} and \cite{bercu2001large} studied large deviation rates of the least squares estimator in an AR(1) process. These contributions provide problem independent bounds, and do not generalize to AR($n$) processes. A problem dependent finite-time deviation and variance bound was provided by \cite{gonzalez2020finite} for stable and unstable AR(1) processes. Unfortunately, the tools used in that work cannot be extended to a multivariate setting. Asymptotic properties of AR($n$) models were obtained by \cite{lai1983asymptotic}.

In a broader context, one of the first non-asymptotic results in system identification was presented in \cite{campi2002finite}, where a uniform bound for the difference between empirical and theoretical identification costs was obtained. More recently, among works that have analyzed finite-time identification for stochastic processes are \cite{jedra2020finite,sarkar2019near,simchowitz2018learning,faradonbeh2018finite} and \cite{zheng2018finite}. These contributions consider state-space formulations with first order vector autoregressive models that, contrary to the description of an AR($n$) process in state-space, normally assume that the noise process perturbs all states instead of only one. In particular, the performance bounds in \cite{simchowitz2018learning} consider the estimation of the full transition matrix instead of the parameters of interest for AR$(n)$ modeling. By leveraging the direct relationship between the coefficients of interest and the transition matrix of the underlying state-space in controller form, bounds for AR($n$) processes could possibly be obtained by finding the (finite-time) optimal projection of the $A$ matrix whose error in operator norm is bounded in \cite{simchowitz2018learning}, such that the resulting matrix is exactly the one provided by the LS estimate of the underlying AR($n$) process. After this, a concentration inequality that bounds the error over $\hat{A}(T)$ and the transition matrix in controller form of the AR($n$) process would be needed. Although plausible, the results we seek do not seem direct from \citep{simchowitz2018learning}. Instead, our analysis resembles that of \cite[Theorem 1]{sarkar2019near} in the derivation of a matrix concentration bound, and similarly to the works cited above, a Gramian matrix associated with the real process dictates the learning rate.

\section{Notation}
\label{sec:notation}
Given a matrix $A\in\mathbb{R}^{n\times n}$, $\rho(A)$ denotes its spectral radius. If $x$ is a vector and $A$ is a fixed positive definite matrix, then $\|x\|_2$ and $\|x\|_\infty$ denote the 2 and $\infty$-norm of $x$, while $\|x\|_A$ is the weighted 2-norm (i.e., $\|x\|_A:=\sqrt{x^\top A x}$). If $B$ is an event, $B^c$ denotes its complement, and $\mathbb{P}(B)$ refers to its probability of occurrence. $\mathbb{E}\{y\}$ denotes the expected value of the random variable $y$.

\section{Problem formulation and preliminary result}
\label{sec:problemformulation}
Consider the following AR$(n)$ process described by
\begin{equation}
\label{arnprocess}
y_t = Y_{t-1}^\top\theta^0 + e_t,
\end{equation}
where $Y_{t-1}^\top := [y_{t-1}\hspace{0.2cm} y_{t-2}\hspace{0.2cm} \dots\hspace{0.2cm} y_{t-n}]$, $e_t$ is a Gaussian white noise of variance $\sigma^2$, and the parameter vector is $\theta^0:=[\theta_1 \hspace{0.2cm} \theta_2 \hspace{0.2cm} \dots\hspace{0.2cm} \theta_n]^\top \in \mathbb{R}^n$. Furthermore, assume that $\{y_t\}$ is a stationary process, and that $\theta^0$ is such that the AR$(n)$ process is asymptotically stationary, which implies that $p(x)=x^n - \theta_1 x^{n-1}-\dots-\theta_{n-1}x-\theta_n$ is a Schur polynomial. In this work, we are interested in how $w^\top \hat{\theta}_N$ concentrates around its true value $w^\top \theta^0$, where $w \in \mathbb{R}^n$ with $\|w\|_2=1$ is fixed and $\hat{\theta}_N$ is the least squares estimator of $\theta^0$ given the data $\{y_t\}_{t=1}^N$. By allowing $w$ to be chosen freely, we study deviation probabilities for single parameters, or linear combinations of them. Note that this probability depends on the true parameters, an thus it gives information about how easily the parameters can be identified through ordinary least squares for a particular system. In other words, our interest is in interpretability; in particular, we are concerned on how the least squares estimator performs under different AR processes.

Unfortunately, an explicit expression of the deviation probability (or equivalently, the confidence region) of interest is elusive in the literature. Therefore, it is of our interest to find an upper bound of it instead. If we define $Y:=[Y_n \hspace{0.2cm} \dots \hspace{0.2cm} Y_{N-1}]^\top$ and $E:= [e_{n+1} \hspace{0.2cm}\dots \hspace{0.2cm} e_{N}]^\top$, we can write $w^\top (\hat{\theta}_N-\theta^0)$ as $w^\top (Y^\top Y)^{-1} Y^\top E$, and hence we pursue a bound of the form
\begin{equation}
\label{prob2}
\mathbb{P}(|w^\top (Y^\top Y)^{-1} Y^\top E|>\varepsilon)\leq \delta,  
\end{equation}
where $\varepsilon$ can be expressed as a function of $\delta,N$, and the true parameters. Note that the stochastic quantity $w^\top(\hat{\theta}_N-\theta^0)$ is a self-normalized process. That is, it is unit free and therefore not affected by scale changes \citep{pena2008self}. These processes are now ubiquitous in the machine learning community, as they arise naturally in, e.g., finite-time analysis of linear systems \citep{simchowitz2018learning} and stochastic bandit problems \citep{krishnamurthy2018semiparametric}.

To derive a bound like \eqref{prob2}, we make use of a martingale tail inequality introduced in \cite{abbasi2011improved}, which is valid for sub-Gaussian stochastic processes.
 
\begin{proposition}[\cite{abbasi2011improved}]
	\label{proposition1}
Let $\{\mathcal{F}_t\}_{t=0}^\infty$ be a filtration. Let $\{\eta_t\}_{t=1}^\infty$ be a real-valued stochastic process such that $\eta_t$ is $\mathcal{F}_t$-measurable and $\eta_t$ is conditionally $R$-sub-Gaussian for some $R> 0$, i.e.
\begin{equation}
\forall \lambda \in \mathbb{R} \quad \mathbb{E}[e^{\lambda \eta_t}| \mathcal{F}_{t-1}] \leq \exp\left( \frac{\lambda^2 R^2}{2} \right). \notag
\end{equation}
Let $\{X_t\}_{t=1}^\infty$ be an $\mathbb{R}^d$-valued stochastic process such that $X_t$ is $\mathcal{F}_{t-1}$-measurable. Assume that $V$ is a $d\times d$ positive definite matrix. For any $t\geq 1$, define
\begin{equation}
\label{VtandSt}
\overline{V}_t = V + \sum_{s=1}^t X_s X_s^\top, \quad \quad S_t = \sum_{s=1}^t \eta_s X_s.
\end{equation}
Then, for any $\delta>0$, with probability at least $1-\delta$, for all $t\geq 1$,
\begin{equation}
\|S_t\|_{{\overline{V}_t}^{-1}}^2 \leq 2R^2 \log \left( \frac{\det(\overline{V}_t)^{1/2}\det(V)^{-1/2}}{\delta} \right). \notag 
\end{equation}
\end{proposition}
Although the result of Proposition \ref{proposition1} is also a deviation bound similar to \eqref{prob2}, it relies on the fact that the matrix $V$ is positive definite, which is valid only for regularized least squares problems. Despite this, in $\overline{V}_t$ we recognize the sample covariance matrix $\sum_{s=1}^t X_s X_s^\top$, which plays an important role in our main result. The key idea behind the proposed approach is to first obtain a finite-sample probability bound on the matrix $Y^\top Y = \sum_{s=n}^{N-1} Y_s Y_s^\top$, and use this result together with Proposition \ref{proposition1} to derive the novel deviation bound. 
%ARXIV VERSION
%In the main text, we provide a proof sketch of Lemma \ref{lemma4} only, in which we use a martingale concentration inequality from \cite{simchowitz2018learning}, and exploit the Gaussianity of $\{e_t\}$ by applying a concentration inequality for $\chi^2$ random variables found in \cite{laurent2000adaptive}.
%CONFERENCE VERSION

\vspace{-0.4cm}
\section{Main results}
\label{sec:mainresults}
In this section, we present our finite-time bounds for AR processes. Firstly, in Theorem \ref{theorem5} we derive a $1-\delta$ concentration bound for the sample covariance matrix $Y^\top Y$, and then use this result to obtain a deviation bound for the least squares estimator for general AR$(n)$ processes, which is presented in Theorem \ref{theorem6}. For the following, we express the AR($n$) process in a state-space formulation
\begin{equation}
\label{ss1} 
x_{t+1}=
\begin{bmatrix}
{\theta^0}^\top & 0 \\
I_n & 0_{n\times 1}
\end{bmatrix}x_t+\begin{bmatrix}
1 \\ 0_{n\times 1}
\end{bmatrix} e_{t+1}, \quad Y_t = \begin{bmatrix}
I_n & 0_{n\times 1}
\end{bmatrix}x_t,  
\end{equation}
where we denote from now on the transition matrix and input vector in \eqref{ss1} as $A$ and $B$ respectively.

\begin{theorem}
	\label{theorem5}
	Consider the AR($n$) process described in \eqref{arnprocess}, and $Y=[Y_n \hspace{0.2cm} \dots \hspace{0.2cm} Y_{N-1}]^\top$. Given $\epsilon>0$, define the following quantities:
	\begin{align}
	\label{barV}
	\overline{V} :&= \sigma^2\sum_{i=0}^\infty A^i BB^\top (A^\top)^i, \\
	V_{\textnormal{dn}}:&= (N-n)\begin{bmatrix} I_n & 0 \end{bmatrix} \left(\overline{V} - \epsilon \sigma^2 \sum_{i=0}^{\infty} A^i (A^\top)^i\right) \begin{bmatrix}
	I_n \\ 0 \end{bmatrix},\notag \\
	V_{\textnormal{up}}:&= (N-n)\begin{bmatrix} I_n & 0 \end{bmatrix} \left(\overline{V} + \epsilon \sigma^2 \sum_{i=0}^{\infty} A^i (A^\top)^i\right) \begin{bmatrix}
	I_n \\ 0 \end{bmatrix}, \notag \\
	\delta(\epsilon,N):&=2\Bigg\{\sqrt{2}\exp\left(\frac{-(N-n)\sigma^2 \epsilon}{24n\mathbb{E}\{y_1^2\}}\right)+\exp\left[-\frac{N-n}{2}\left(1+\frac{\epsilon}{3}-\sqrt{1+\frac{2\epsilon}{3}}\right)  \right]\notag \\
	\label{deltadef}
	&\qquad \quad +\exp\left[-\frac{(N-n)\epsilon}{72 (\|\theta\|_2+1)^2 \tilde{\beta}}\right] + \exp(-\epsilon \sqrt{N})\Bigg\},
	\end{align}
	where
	\vspace{-0.4cm}
	\begin{align}
	\label{MPhi}
	M_\Phi :&=  \max_\omega |e^{j\omega n}- \theta_1 e^{j\omega(n-1)}-\dots-\theta_n|^{-2}, \\
	\label{tildebeta}
	\tilde{\beta} :&= \frac{(n+1)N}{N-n}\left[\frac{\mathbb{E}\{y_1^2\}}{\epsilon \sigma^2}+\frac{2M_\Phi(1+\epsilon^{-1/2})}{N^{1/4}}\right].
	\vspace{-0.1cm}
	\end{align}
	\vspace{-0.2cm}
	Then, for all $\epsilon>0$ such that $V_{\textnormal{dn}}\succ 0$, we have
	\begin{equation}
	\label{probthm5}
	\mathbb{P}(V_{\textnormal{dn}} \preceq Y^\top Y \preceq V_{\textnormal{up}}) \geq 1-\delta(\epsilon,N).
	\end{equation}
\end{theorem}
\vspace{-0.1cm}
\begin{proof}
	The AR$(n)$ process can be rewritten as in \eqref{ss1}, where $x_t$ is equal to $\begin{bmatrix}
	Y_{t}^\top & y_{t-n}
	\end{bmatrix}^\top$. Note that the eigenvalues of $A$ are precisely the poles of the autoregressive process, with an extra eigenvalue at $0$. We are interested in bounding
	\vspace{-0.2cm}
	\begin{equation}
	Y^\top Y= \begin{bmatrix}
	I_n & 0
	\end{bmatrix} \sum_{i=n}^{N-1} x_i x_i^\top  \begin{bmatrix}
	I_n \\ 0 \notag
	\end{bmatrix}.\vspace{-0.1cm}
	\end{equation}
	The approach consists in first determining a concentration bound for $\sum_{i=n}^{N-1}x_i x_i^\top$, and then relating it to a concentration bound for $Y^\top Y$. In this spirit, we write
	\vspace{-0.1cm}
	\begin{equation}
	\label{summingup}
	x_{i+1} x_{i+1}^\top = Ax_ix_i^\top A^\top + Ax_i e_{i+1} B^\top + Be_{i+1} x_i^\top A^\top + e_{i+1}^2B B^\top,
	\end{equation}
	and denote $V_N$ as $(N-n)^{-1}\sum_{i=n}^{N-1} x_i x_i^\top$. If we sum over $i=n-1,\dots,N-2$ in \eqref{summingup}, we obtain
\small{
	\begin{equation}
	V_N \hspace{-0.05cm}= \hspace{-0.05cm} AV_N A^\top \hspace{-0.05cm} +\hspace{-0.02cm}\underbrace{\frac{1}{N-n}\hspace{-0.06cm}\left[A(x_{n-1} x_{n-1}^\top\hspace{-0.05cm}-\hspace{-0.05cm}x_{N-1}x_{N-1}^\top) A^\top\hspace{-0.1cm}+\hspace{-0.15cm} \sum_{i=n-1}^{N-2}\hspace{-0.16cm}\left(e_{i+1} Ax_i B^\top \hspace{-0.12cm}+ \hspace{-0.03cm}e_{i+1} B x_i^\top A^\top \hspace{-0.1cm}+\hspace{-0.02cm} e_{i+1}^2 B B^\top\right)\right]}_{E_N}\hspace{-0.08cm}. \notag 
\end{equation}}
\normalsize
\hspace{-0.1cm}Since the AR$(n)$ process is asymptotically stationary, the Lyapunov equation above has as solution $V_N = \sum_{i=0}^\infty A^i E_N (A^\top)^i$.
	By construction, $V_N$ tends to $\overline{V}$ (defined in \eqref{barV}) with probability 1 as $N$ tends to infinity (see, e.g. \citep[p. 64]{soderstrom2002discrete}). Thus, our goal is to obtain a finite-sample concentration bound that relates $V_N$ with $\overline{V}$. For this, we bound $\sum_{i=n-1}^{N-2}e_i^2$ by its variance, and bound the other terms of $E_N$ by a small matrix quantity $\epsilon I$, for all $N>N(\epsilon)$. Lemmas \ref{lemma2}, \ref{lemma3} and \ref{lemma4} are needed for this purpose, which bound the probability of the following events:
	\begin{align}
	\mathcal{E}_1 &:= \left\{\rho\left[A(x_{n-1}x_{n-1}^\top - x_{N-1}x_{N-1}^\top)A^\top\right] \leq \epsilon \sigma^2(N-n)/3\right\}, \notag \\
	\mathcal{E}_2 &:= \left\{\rho\left[\textstyle\sum_{i=n-1}^{N-2} e_{i+1}^2 BB^\top - BB^\top \right] \leq \epsilon\sigma^2(N-n)/3\right\}, \notag \\
	\mathcal{E}_3 &:= \left\{\rho\left[\textstyle\sum_{i=n-1}^{N-2}\left( e_{i+1} Ax_i B^\top + e_{i+1} B x_i^\top A^\top\right)\right] \leq \epsilon\sigma^2(N-n)/3\right\}. \notag
	\end{align}
	\vspace{-0.6cm}
	\begin{lemma} 
		\label{lemma2}Consider the process described in \eqref{ss1}, where $\{e_i\}$ is a Gaussian zero-mean i.i.d. of variance $\sigma^2$, and $\{x_t\}$ is a stationary random process. Then,
		\begin{equation}
		\mathbb{P}\left(\mathcal{E}_1\right) \geq 1-2\sqrt{2}\exp\left(\frac{-(N-n) \sigma^2 \epsilon }{24n\mathbb{E}\{y_1^2\}}\right). \notag 
		\end{equation}
	\end{lemma}
		\vspace{-0.4cm}
	\begin{lemma}
		\label{lemma3}
		Let $\{e_i\}$ be a Gaussian zero-mean i.i.d. sequence of variance $\sigma^2$. Then,
		\begin{equation}
		\mathbb{P}\left( \mathcal{E}_2 \right)\geq 1-2\exp\left[-\frac{N-n}{2}\left(1+\frac{\epsilon}{3}-\sqrt{1+\frac{2\epsilon}{3}}\right)  \right]. \notag 
		\end{equation}
	\end{lemma}
	\begin{lemma} 
		\label{lemma4}
		Consider the same assumptions as in Lemma \ref{lemma2}. For any $\epsilon>0$, we have
		\begin{equation}
		\mathbb{P}\left(\mathcal{E}_3\right) \geq 1- 2\exp\left[-\frac{(N-n)\epsilon}{72 (\|\theta^0\|_2+1)^2 \tilde{\beta}}\right] - 2\exp(-\epsilon \sqrt{N}), \notag 
		\end{equation}
		where $M_\Phi$ and $\tilde{\beta}$ are defined in \eqref{MPhi} and \eqref{tildebeta} respectively.
	\end{lemma}
In the main text, we provide proof of Lemma \ref{lemma4} only, which can be found in Section \ref{seclemma}. With these three lemmas, and by the subadditivity of the spectral radius of Hermitian matrices \cite[Fact 5.12.2]{bernstein2009matrix}, we have
	\begin{align}
	\mathcal{E}_1 \cap \mathcal{E}_2 \cap \mathcal{E}_3 \implies \rho(E_N - \sigma^2 BB^\top)\leq \sigma^2\epsilon \implies \sigma^2 (BB^\top-\epsilon I) \preceq E_N \preceq \sigma^2 (BB^\top+\epsilon I), \notag
	\end{align}
	which occurs with probability not less than $1-\delta(\epsilon,N)$. This also implies \eqref{probthm5}.
\end{proof}

Theorem \ref{theorem5} delivers a finite-sample bound on the sample covariance matrix. Naturally, this matrix will deviate from its expected value by a small amount for large sample sizes. Note that this bound depends on a fixed value $\epsilon$, which can be chosen arbitrarily small. As most self-normalized process bounds, $\delta(\epsilon,N)$ does not depend on the variance of the process noise. 

With this result, we are ready to state the desired deviation bound in Theorem \ref{theorem6}.

\begin{theorem}
	\label{theorem6}
    Consider the AR($n$) process described in \eqref{arnprocess}, where $\theta^0$ is assumed to yield an asymptotically stationary process, $\{e_t\}$ is an i.i.d. Gaussian random process with variance $\sigma^2$, and $\{y_t\}$ is stationary. Then, 
	\begin{equation}
	\label{probarn}
	\mathbb{P}\left(|w^\top (\hat{\theta}_N-\theta^0)| > 2\sigma \| w^\top V_{\textnormal{dn}}^{-1/2} \|_2 \sqrt{\log\left( \frac{\det(V_{\textnormal{up}}V_{\textnormal{dn}}^{-1} + I_n)^{1/2}}{\delta(\epsilon,N)} \right)}\right)\leq 2\delta(\epsilon,N),
	\end{equation}	
	where $V_{\textnormal{dn}}$, $V_{\textnormal{up}}$ and $\delta(\epsilon,N)$ are as described in Theorem \ref{theorem5}.
\end{theorem}

\begin{proof}
	We will follow the main ideas in \cite[Theorem 1]{sarkar2019near}. We start by writing an upper bound using the Cauchy-Schwartz inequality
	\begin{equation}
	|w^\top (\hat{\theta}_N-\theta^0)| = |w^\top (Y^\top Y)^{-1}Y^\top E| \leq \| w^\top(Y^\top Y)^{-1/2} \|_2 \|(Y^\top Y)^{-1/2}Y^\top E \|_2. \notag
	\end{equation}
	In Theorem \ref{theorem5} we have found deterministic matrices $V_{\textnormal{dn}}, V_{\textnormal{up}}$ and a scalar $\delta(\epsilon,N)$ such that, for the event $\mathcal{E}_{\textnormal{pm}} := \{V_{\textnormal{dn}} \preceq Y^\top Y \preceq V_{\textnormal{up}}\}$, we have $	\mathbb{P}(\mathcal{E}_\textnormal{pm})\geq 1-\delta(\epsilon,N).$ The next step is to bound the self-normalized norm. This can be done by first defining the event
	\begin{equation}
	\mathcal{E}_{\textnormal{sn}} := \left\{  \| Y^\top E \|_{(Y^\top Y + V_{\textnormal{dn}})^{-1}} \leq \sqrt{2\sigma^2 \log\left( \frac{\det(Y^\top Y + V_{\textnormal{dn}})^{1/2}\det(V_{\textnormal{dn}})^{-1/2}}{\delta(\epsilon,N)} \right)} \right\}. \notag 
	\end{equation}
	It follows from Proposition \ref{proposition1} that $\mathbb{P}(\mathcal{E}_{\textnormal{sn}}) \geq 1-\delta(\epsilon,N)$. Also, under $\mathcal{E}_{\textnormal{pm}}$ we have that $Y^\top Y + V_{\textnormal{dn}} \preceq 2Y^\top Y$, which implies $(Y^\top Y + V_{\textnormal{dn}})^{-1} \succeq \frac{1}{2}(Y^\top Y)^{-1}$. So, considering the set $\mathcal{E}_{\textnormal{pm}}  \cap \mathcal{E}_{\textnormal{sn}}$, we obtain
	\begin{equation}
	\mathcal{E}_{\textnormal{pm}}  \cap \mathcal{E}_{\textnormal{sn}} \implies \mathcal{E}_{\textnormal{pm}}  \cap \left\{ \|(Y^\top Y)^{-1/2} Y^\top E \|_2 \leq 2\sigma \sqrt{\log\left( \frac{\det(V_{\textnormal{up}}V_{\textnormal{dn}}^{-1} + I_n)^{1/2}}{\delta(\epsilon,N)} \right)}  \right\}. \notag 
	\end{equation}
	Furthermore, observe that $\mathbb{P}(\mathcal{E}_{\textnormal{pm}}  \cap \mathcal{E}_{\textnormal{sn}})\geq 1-2\delta(\epsilon,N)$. So, if $\mathcal{E}_{\textnormal{pm}}  \cap \mathcal{E}_{\textnormal{sn}}$ holds, then 
	\begin{equation}
	\label{eventthm6}
	|w^\top (Y^\top Y)^{-1}Y^\top E| \leq 2\sigma \| w^\top V_{\textnormal{dn}}^{-1/2} \|_2 \sqrt{\log\left( \frac{\det(V_{\textnormal{up}}V_{\textnormal{dn}}^{-1} + I_n)^{1/2}}{\delta(\epsilon,N)} \right)}, 
	\end{equation}
	which means that the probability of the event in \eqref{eventthm6} is at least $1-2\delta(\epsilon,N)$. By considering the complement event, we obtain the probability bound \eqref{probarn}.
\end{proof}

Theorem \ref{theorem6} provides a finite-sample confidence bound on the deviation of the weighted parameter vector $w^\top \hat{\theta}_N$ with respect to its asymptotic value $w^\top \theta^0$. This result delivers probability bounds on the deviation each parameter $\theta_i$ individually, as well as any linear combination of them. Note that $\epsilon$ can be considered a \textit{tightness variable}, as by setting $\epsilon$ small, more samples are required to guarantee a desired confidence level, but the probability bound will be tighter. 

To end this analysis, we derive the decay rate of our probability bound in Corollary \ref{corollaryrate}.

\begin{corollary}
	\label{corollaryrate}
If $\epsilon$ is picked as $\lambda_n-N^{-1/2}$, where $\lambda_n$ is the smallest eigenvalue of 
\begin{equation}
\left(\sigma^2 \begin{bmatrix}
I_n \\ 0
\end{bmatrix}^\top \sum_{i=0}^\infty A^i (A^\top)^i\begin{bmatrix}
I_n \\ 0
\end{bmatrix}\right)^{-1/2} 
\begin{bmatrix}
I_n \\ 0
\end{bmatrix}^\top \overline{V}\begin{bmatrix}
I_n \\ 0
\end{bmatrix}
\left(\sigma^2 \begin{bmatrix}
I_n \\ 0
\end{bmatrix}^\top \sum_{i=0}^\infty A^i (A^\top)^i\begin{bmatrix}
I_n \\ 0
\end{bmatrix}\right)^{-\top/2}, \notag
\end{equation}
then $\delta \sim C e^{-\lambda_n \sqrt{N}}$ for large $N$ and the deviation in \eqref{probarn} is asymptotically a constant in $N$. This shows that the rate of decay of the probability bound is at least exponential in $\sqrt{N}$.
\end{corollary}

\section{Proof of Lemma \ref{lemma4}}
\label{seclemma}
Here we present a sketch of the proof of Lemma \ref{lemma4}, in which we use a martingale concentration inequality from \cite{simchowitz2018learning} and exploit the Gaussianity of $\{e_t\}$ by applying a concentration inequality for $\chi^2$ random variables found in \cite{laurent2000adaptive}.
\begin{proof}\textbf{of Lemma \ref{lemma4}}
	For any vector $q:=\begin{bmatrix}q_1 & \tilde{q}^\top
	\end{bmatrix}^\top \in \mathbb{R}^{n+1}$ of unit 2-norm, we have
	\vspace{-0.2cm}
	\begin{equation}
	\label{dotprod}
	q^\top \left(\frac{1}{N-n}\sum_{i=n}^{N-1} e_i(A x_{i-1}B^\top+ B x_{i-1}^\top A^\top) \right)q=\frac{2q_1}{N-n}(q_1{\theta^0}^\top+\tilde{q}^\top)\sum_{i=n}^{N-1}e_i Y_{i-1}.
	\end{equation}
	Since \eqref{dotprod} is symmetric around zero, it is sufficient to bound its upper tail. Next, we denote the vector $z:=2q_1(q_1\theta^0 + \tilde{q})$. By using Lemma 4.2 of \cite{simchowitz2018learning}, with $Z_t = \frac{1}{\sqrt{N-n}}\langle z, Y_{t-1}\rangle$, $W_t = e_t$, and $\beta = \epsilon \tilde{\beta}(N-n)\sigma^2\max_{\|q\|_2=1}\|2q_1{\theta^0}^\top+\tilde{q}^\top\|_2^2$, we obtain the inequality
	\vspace{-0.2cm}
	\small{
		\begin{equation}
		\mathbb{P}\left[\left\{ \sum_{i=n}^{N-1} \frac{\langle z, Y_{i-1}\rangle e_i}{N-n} \geq \frac{\sigma^2\epsilon}{3} \right\} \cap \left\{\sum_{i=n}^{N-1} \frac{\|Y_{i-1}\|_2^2}{\sigma^2(N-n)} \leq \epsilon \tilde{\beta} \right\}  \right] \leq \exp\left(-\frac{(N-n)\epsilon}{72\tilde{\beta} \max\limits_{\|q\|_2=1}\|q_1(q_1\theta^0 + \tilde{q})\|_2^2}\right). \notag
		\end{equation}} 
	\normalsize 
	Using the well-known inequality $\mathbb{P}(A\cap B)\geq \mathbb{P}(A)-\mathbb{P}(B^c)$, and the fact that $\sum_{i=n}^{N-1} \|Y_{i-1}\|_2^2\leq (n+1)\sum_{i=-1}^{N-2} y_i^2$, we obtain
	\vspace{-0.2cm}
	\small{\begin{equation}
		\mathbb{P}\left(\sum_{i=n}^{N-1} \frac{\langle z, Y_{i-1}\rangle e_i}{N-n} \geq \frac{\sigma^2\epsilon}{3} \right) \leq \exp\left(-\frac{(N-n)\epsilon}{72 \max\limits_{\|q\|_2=1}\|q_1(q_1\theta^0 + \tilde{q})\|_2^2 \tilde{\beta}}\right) + \mathbb{P}\left(\frac{n+1}{\sigma^2(N-n)}\sum_{i=-1}^{N-2} y_i^2 > \epsilon \tilde{\beta}\right). \notag
		\end{equation}}
	\normalsize 
	\hspace{-0.1cm}To tackle the last probability, we note that $\begin{bmatrix}
	y_{-1} & \cdots & y_{N-2}
	\end{bmatrix}^\top \sim \mathcal{N}(0,R_{N})$, where $R_{N}$ is a symmetric Toeplitz covariance matrix of eigenvalues $\{\lambda_i\}_{i=1}^N$. Hence, $Z = \sum_{i=-1}^{N-2}y_i^2$ is a generalized $\chi^2$ random variable, whose distribution is equal to the distribution of $v^\top R_N v$, where $v\sim \mathcal{N}(0,I_N)$. By the singular value decomposition $R_N = U_ND_NU_N^\top$ where $D_N = \textnormal{diag}\{\lambda_i\}$ and $U_N$ is a unitary matrix, and the rotation invariance of $v$ \cite[Chap. 3]{vershynin2018high}, we see that
	\vspace{-0.2cm}
	\begin{equation}
	\mathbb{P}\left(\sum_{i=-1}^{N-2} y_i^2 - \mathbb{E}\left\{\sum_{i=-1}^{N-2} y_i^2\right\} >t  \right) = \mathbb{P}\left(\sum_{i=1}^{N} \lambda_i (\tilde{v}_i^2-1) > t \right), \notag 
	\end{equation}
	where $\tilde{v}\sim \mathcal{N}(0,I_N)$. Then, by \cite[Lemma 1]{laurent2000adaptive},
	\vspace{-0.2cm}
	\begin{equation}
	\label{fromlaurent}
	\mathbb{P}\left(\sum_{i=-1}^{N-2} y_i^2 - \mathbb{E}\left\{\sum_{i=-1}^{N-2} y_i^2\right\} >2\|\lambda\|_2\sqrt{t} + 2 \|\lambda\|_\infty t  \right) \leq \exp(-t). 
	\end{equation}
	It is known (see, e.g. \cite[Section 4.2]{gray2006toeplitz}) that the maximum eigenvalue of $R_N$ is bounded by $\sigma M_\Phi$, where $M_\Phi$ is defined as in \eqref{MPhi}. By letting $t=\epsilon \sqrt{N}$ in \eqref{fromlaurent}, and upper bounding $\|\lambda\|_2$ and $\|\lambda\|_\infty$ by $\sqrt{N}M_\Phi$ and $M_\Phi$ respectively, we deduce that
	\vspace{-0.1cm}
	\begin{equation}
	\mathbb{P}\left(\frac{n+1}{\sigma^2(N-n)} \sum_{i=-1}^{N-2} y_i^2 > \underbrace{\frac{(n+1)N}{N-n}\left[\frac{\mathbb{E}\{y_1^2\}}{\sigma^2}+\frac{2M_\Phi(\epsilon+ \sqrt{\epsilon})}{N^{1/4}}\right]}_{\epsilon\tilde{\beta}}  \right)\leq \exp(-\epsilon \sqrt{N}). \notag 
	\end{equation}
	Finally, note that $\max_{\|q\|_2=1}\|q_1(q_1\theta^0 + \tilde{q})\|_2^2 \leq (\|\theta^0\|_2+1)^2$. With this, and considering the complement event, we reach the bound in Lemma \ref{lemma4}.
\end{proof}
	\vspace{-0.1cm}

\section{Discussion and conclusions}
\label{sec:discussion}
In this paper, we have provided finite-sample guarantees for the least squares estimates of the coefficients of general AR($n$) processes. For this, a concentration bound for the sample covariance matrix was derived. In this bound, the Gramian matrix $\sum_{i=0}^\infty A^i (A^\top)^i$ in $V_{\textnormal{dn}}$ and $V_{\textnormal{up}}$ shows that faster processes need less samples to guarantee concentration of the covariance matrix, which is a natural result. Regarding Theorem \ref{theorem6}, we find that the fixed vector $w$ impacts the confidence bound through the inverse of $V_{\textnormal{dn}}$, which resembles the results obtained in \cite[Eq. 20.2]{lattimore2018bandit} for least squares estimates of linear bandit algorithms with deterministic actions. The $\log\det$ term is also unsurprising, as it also appears in finite-sample analysis of LTI systems (see, e.g. \cite[Eq. 12]{sarkar2019near}). The deterministic matrices $V_{\textnormal{dn}}$ and $V_{\textnormal{up}}$ in \eqref{probarn} capture the correct behavior of the confidence bound, since it is large when the uncertainty on the sample covariance matrix is also large. Also, note that the proof of Theorem \ref{theorem6} heavily relies on bounding the probability of the normal matrix $Y^\top Y$, but it is easily decoupled from Theorem \ref{theorem5}. That is, if tighter bounds for $\mathcal{E}_{\textnormal{pm}}$ can be found, then Theorem \ref{theorem6} can be directly improved. Future work concerns proving finite-time variance bounds for the estimated parameters, extending the analysis for ARX models under sub-Gaussian noise, and deriving sharp lower bounds for AR($n$) processes.

\subsection*{Acknowledgments}
This work was supported by the Swedish Research Council under contract number 2016-06079 (NewLEADS).
\bibliography{References}

\newpage
\appendix 
\section{Proofs of Lemmas \ref{lemma2} and \ref{lemma3}}

\textbf{Proof of Lemma \ref{lemma2}.} For simplicity we consider $\tilde{\epsilon}:=\sigma^2 \epsilon$. We first see that
\begin{align}
\mathbb{P}&\left(\rho\left[\frac{1}{N-n}A(x_{n-1}x_{n-1}^\top - x_{N-1}x_{N-1}^\top)A^\top\right] >\frac{\tilde{\epsilon}}{3}\right)  \notag \\
&\hspace{1cm}\leq \mathbb{P}\left(\rho\left[\frac{1}{N-n}A(x_{n-1}x_{n-1}^\top + x_{N-1}x_{N-1}^\top)A^\top\right] >\frac{\tilde{\epsilon}}{3} \right) \notag \\
&\hspace{1cm}\leq \mathbb{P}\left( \frac{\|Ax_{n-1} \|_2^2}{N-n} + \frac{\|Ax_{N-1} \|_2^2}{N-n}>\frac{\tilde{\epsilon}}{3}  \right) \notag \\
&\hspace{1cm}\leq \mathbb{P}\left( \frac{\|x_{n-1} \|_2^2}{N-n} + \frac{\|x_{N-1} \|_2^2}{N-n}>\frac{\tilde{\epsilon}}{3}  \right) \notag \\
&\hspace{1cm}\leq \mathbb{P}\left( \frac{\| x_{n-1}\|_2^2}{N-n}>\frac{\tilde{\epsilon}}{6} \cup  \frac{\|x_{N-1}\|_2^2}{N-n}>\frac{\tilde{\epsilon}}{6}  \right).\notag 
\end{align}
By stationarity, the probabilities of the events 
\begin{equation}
\mathcal{E}_n = \left\{ \frac{\| x_{n-1}\|_2^2}{N-n}>\frac{\tilde{\epsilon}}{6}\right\}, \quad \mathcal{E}_N = \left\{ \frac{\| x_{N-1}\|_2^2}{N-n}>\frac{\tilde{\epsilon}}{6}\right\}, \notag
\end{equation}
are equal. For some fixed $s>0$, we can bound $\mathbb{P}(\mathcal{E}_n)$ by Chernoff's inequality \cite[Chap. 2]{wainwright2019high} and later compute the moment generating function of the generated $\chi^2$ distribution:
\begin{equation}
\mathbb{P}\left(\mathcal{E}_n\right)\leq \exp\left(-s \frac{\tilde{\epsilon}}{6}\right) \mathbb{E}\left[ \exp\left( \frac{s}{N-n} \sum_{t=-1}^{n-1} y_t^2 \right)\right] = \exp\left(-s \frac{\tilde{\epsilon}}{6}\right) \frac{1}{\prod_{i=1}^{n}\sqrt{1-\frac{2s}{N-n}\lambda_i(R_n)}}, \notag 
\end{equation}
where $R_n \in \mathcal{S}^n$ is the Toeplitz covariance matrix of $[y_{-1} \hspace{0.15cm} \cdots \hspace{0.15cm} y_{n-1}]$, $\lambda_i(R_n)$ is the $i$-th eigenvalue of the $R_n$, and $0<s<(N-n)/(2\lambda_{\textnormal{max}}(R_n))$. By using the Weierstrass product inequality (stated as Proposition \ref{weierstrassprodinequality} in the Auxiliary Results section), we obtain
\begin{equation}
\mathbb{P}\left(\mathcal{E}_n \right) \leq \exp\left(-s \frac{\tilde{\epsilon}}{6}\right) \left(1-\frac{2sn\mathbb{E}\{y_1^2\}}{N-n}\right)^{-1/2}. \notag 
\end{equation}
By noting that
\begin{equation}
\frac{N-n}{2\lambda_{\textnormal{max}}(R_n)}> \frac{N-n}{2\textnormal{tr}(R_n)}>\frac{N-n}{4n\mathbb{E}\{y_1^2\}}>0, \notag 
\end{equation}
we set $s=(N-n)/(4n\mathbb{E}\{y_1^2\})$ to derive the exponential inequality
\begin{equation}
\mathbb{P}\left(\frac{\| x_{n-1}\|_2^2}{N-n}>\frac{\tilde{\epsilon}}{6}\right) \leq \sqrt{2}\exp\left(\frac{-(N-n) \tilde{\epsilon}}{24n\mathbb{E}\{y_1^2\}}\right). \notag 
\end{equation}
So, by the chain of inequalities above, we conclude that
\begin{equation}
\mathbb{P}\left(\rho\left[\frac{1}{\sigma^2(N-n)}A(x_{n-1}x_{n-1}^\top - x_{N-1}x_{N-1}^\top)A^\top\right] >\frac{\epsilon}{3}\right)\leq 2\sqrt{2}\exp\left(\frac{-(N-n)\sigma^2 \epsilon }{24n\mathbb{E}\{y_1^2\}}\right), \notag 
\end{equation}
which implies the statement we wanted to prove. \hfill $\blacksquare$\\

\hspace{-0.6cm}\textbf{Proof of Lemma \ref{lemma3}.}
Event $\mathcal{E}_2$ is equivalent to the event
\begin{equation}
\left\{ \left|\frac{1}{N-n} \sum_{i=n-1}^{N-2} e_{i+1}^2 - \sigma^2\right| \leq \frac{\sigma^2\epsilon}{3} \right\}. \notag 
\end{equation}
To bound the probability of this event, we shall use a corollary of Lemma 1 of \cite{laurent2000adaptive} (Corollary \ref{laurentcorollary} in the Auxiliary Results section), which gives high probability bounds on the tails of a $\chi^2$ statistic. Via this result and unnormalizing the $\chi^2$ statistic, we obtain that
\begin{equation}
\label{unnormal1}
\mathbb{P}\left(\frac{1}{N-n}\sum_{i=n}^{N-1}e_i^2- \sigma^2 \geq 2\sigma^2\sqrt{\frac{x}{N-n}}+2\sigma^2\frac{x}{N-n}\right)\leq e^{-x}
\end{equation}
and
\begin{equation}
\label{unnormal2}
\mathbb{P}\left( \frac{1}{N-n}\sum_{i=n}^{N-1}e_i^2- \sigma^2 \leq -2\sigma^2\sqrt{\frac{x}{N-n}}\right)\leq e^{-x}. 
\end{equation}
Equations \eqref{unnormal1} and \eqref{unnormal2} imply that
\begin{equation}
\label{eqlemma2}
\mathbb{P}\left(-2\sigma^2\sqrt{\frac{x}{N-n}}\leq \frac{1}{N-n}\sum_{i=n}^{N-1}e_i^2- \sigma^2 \leq 2\sigma^2\sqrt{\frac{x}{N-n}}+2\sigma^2\frac{x}{N-n}\right)\geq 1-2e^{-x}. 
\end{equation}
Next, we are interested in solving the following quadratic equation for positive $x$:
\begin{align}
2\sigma^2 \sqrt{\frac{x}{N-n}}+2\sigma^2 \frac{x}{N-n} = \frac{\sigma^2\epsilon}{3} \implies& \sqrt{\frac{x}{N-n}} = \frac{-1+\sqrt{1+\frac{2\epsilon}{3}}}{2} \notag \\
\implies& x  = \frac{N-n}{2}\left(1+\frac{\epsilon}{3}-\sqrt{1+\frac{2\epsilon}{3}}\right). \notag 
\end{align}
By plugging this value of $x$ in \eqref{eqlemma2}, we obtain the result. \hfill $\blacksquare$

\section{On the decay rate of the bound}
\label{appendixdecay}
Here we analyze the rate of decay of the deviation bound given in Theorem \ref{theorem6}, which is stated explicitly in Corollary \ref{corollaryrate}. In particular, we study how fast $\delta$ decays to zero if the confidence bound is held constant. 

Consider \eqref{probarn}, and denote $\delta':= 2\delta$. Also, denote $\tilde{V}_{\textnormal{dn}}$ as $\frac{1}{N-n} V_{\textnormal{dn}}$, and
\begin{equation}
\varepsilon := 2\sqrt{\log(2)}\sigma \|w^\top V_{\textnormal{dn}}^{-1/2}\|_2 \sqrt{\log\left(\frac{\det(V_{\textnormal{up}}V_{\textnormal{dn}}^{-1}+I_n)^{1/2}}{\delta'}\right)}. \notag 
\end{equation}
% a fixed $\epsilon$, from \eqref{deltadef} we find that $\delta' \sim C e^{-\epsilon \sqrt{N}}$ for large $N$. 

%Now, if $\epsilon$ grows\footnote{This growth will typically not be comparable with $\sqrt{N}$, because $\epsilon$ is upper bounded through the constraint $V_{\textnormal{dn}}\succ 0$.}

First, we analyze $\tilde{V}_{\textnormal{dn}}$. We write $\tilde{V}_{\textnormal{dn}}$ as $\tilde{\overline{V}}-\epsilon \bar{\Gamma}$, where $\tilde{\overline{V}}:=[I_n \hspace{0.2cm} 0] \overline{V} [I_n \hspace{0.2cm} 0]^\top$ and $\bar{\Gamma} := \sigma^2 [I_n \hspace{0.2cm} 0] \sum_{i=0}^{\infty} A^i (A^\top)^i [I_n \hspace{0.2cm} 0]^\top$. Next, let $WW^\top$ be the Cholesky factorization of $\bar{\Gamma}$, and let $UDU^\top$ be the eigenvalue decomposition of $W^{-1}\tilde{\overline{V}} W^{-\top}$, where $D = \textnormal{diag}\{\lambda_1,\dots,\lambda_n\}$, and $\lambda_1\geq \dots \geq \lambda_n>0$. Thus,
\begin{equation}
\tilde{V}_{\textnormal{dn}}^{-1} = W^{-\top}U\begin{bmatrix}
\frac{1}{\lambda_1-\epsilon} & & 0 \\
& \ddots & \\
0& & \frac{1}{\lambda_n-\epsilon} 
\end{bmatrix}U^\top W^{-1}. \notag 
\end{equation}
Next, we choose $\epsilon$ such that $\epsilon = \lambda_n-N^{-1/2}$. This choice leads to 
\begin{equation}
\label{vdn11}
\|w^\top \tilde{V}_{\textnormal{dn}}^{-1/2}\|_2 = \sqrt{w^\top V_1 w + N^{1/2} w^\top V_2 w},
\end{equation}
where $V_1$ and $V_2$ are suitable matrices independent of $N$. In particular,
\begin{equation}
V_2 = W^{-\top}U\begin{bmatrix}
0_{n-k} & 0\\
0 & I_k  
\end{bmatrix}U^\top W^{-1}, \notag 
\end{equation}
with $k$ being the algebraic multiplicity of $\lambda_n$. Also, note that
\begin{equation}
\label{vdn12}
\det(V_{\textnormal{dn}}^{-1}) = \frac{N^{k/2}}{\det(\bar{\Gamma}) \prod_{i=1}^{n-k}(\lambda_i-\epsilon)}.
\end{equation}
Regarding $\delta'$, with the choice of $\epsilon = \lambda_n-N^{-1/2}$ we find that $\delta' \sim C e^{-\lambda_n \sqrt{N}}$ for large $N$. Hence, \eqref{vdn11} and \eqref{vdn12} lead to the following computations:
\begin{align}
	\varepsilon &= 2\sqrt{\log(2)} \sigma \sqrt{\frac{w^\top \tilde{V}_{\textnormal{dn}}^{-1} w}{(N-n)}} \sqrt{\log\left(\frac{\det(V_{\textnormal{up}}+V_{\textnormal{dn}})^{1/2}}{\det(V_{\textnormal{dn}})^{1/2}\delta'}\right)} \notag \\
	&= 2\sqrt{\log(2)} \sigma \sqrt{\frac{w^\top V_1 w + N^{1/2} w^\top V_2 w}{(N-n)}} \sqrt{\log\left(\frac{\det\left(\begin{bmatrix} I_n & 0 \end{bmatrix} \overline{V} \begin{bmatrix} I_n \\ 0 \end{bmatrix}     \right)^{1/2}}{\det(\tilde{V}_{\textnormal{dn}})^{1/2}\delta'}\right)} \notag \\
	&\sim 2\sqrt{\log(2)} \sigma N^{-1/4} \sqrt{w^\top V_2 w} \sqrt{\frac{1}{2}\log \left(\frac{\det\left(\begin{bmatrix} I_n & 0 \end{bmatrix} \overline{V} \begin{bmatrix} I_n \\ 0 \end{bmatrix}\right)}{C^2 \det(\bar{\Gamma}) \prod_{i=1}^{n-k}(\lambda_i-\epsilon)} \right)+\frac{k}{4}\log(N)+\lambda_n \sqrt{N}} \notag \\
	&= \mathcal{O}(1) \textnormal{ as }N\to \infty. \notag 
\end{align}
Hence, the rate of decay is at least exponential in $\sqrt{N}$. That is, for an asymptotically constant confidence bound, the probability of interest decays as $e^{-\lambda_n \sqrt{N}}$. Note that a byproduct of this derivation are the directions of $w$ in which a faster learning can be achieved: Since $V_2$ has rank $k$, a subspace of dimension $n-k$ of directions satisfies $w^\top V_2 w=0$, for which the dominant term in $\|w^\top V_{\textnormal{dn}}^{-1/2}\|_2$ becomes $N^{-1/2}$ instead of $N^{-1/4}$.   

\section{Auxiliary results}
\begin{proposition}[Lemma 4.2 of \cite{simchowitz2018learning}]
Let $\{\mathcal{F}_t\}_{t\geq 0}$ be a filtration, and $\{Z_t\}_{t\geq 1}$ and $\{W_t\}_{t\geq 1}$ be real-valued processes adapted to $\mathcal{F}_t$ and $\mathcal{F}_{t+1}$ respectively. Moreover, assume $W_t|\mathcal{F}_t$ is mean zero and $\sigma^2$-sub-Gaussian. Then, for any positive real numbers $\alpha,\beta$ we have
\begin{equation}
\mathbb{P}\left[\left\{ \sum_{t=1}^T Z_t W_t \geq \alpha \right\} \cap \left\{\sum_{t=1}^T Z_t^2 \leq \beta \right\}  \right] \leq \exp\left(-\frac{\alpha^2}{2\sigma^2 \beta}\right). \notag
\end{equation}
\end{proposition}

\begin{proposition}[Lemma 1 of \cite{laurent2000adaptive}]
	\label{laurentproposition}
Let $(Y_1,\dots,Y_D)$ be i.i.d. Gaussian variables, with mean $0$ and variance $1$. Let $a_1,\dots,a_D$ be nonnegative. We set
\begin{equation}
\|a\|_{\infty} = \sup_{i=1,\dots,D} |a_i|, \quad \|a\|_2^2 = \sum_{i=1}^{D}a_i^2. \notag 
\end{equation}
Let 
\begin{equation}
Z = \sum_{i=1}^D a_i (Y_i^2-1). \notag 
\end{equation}
Then, the following inequalities hold for any positive $x$:
\begin{align}
\mathbb{P}(Z\geq 2\|a\|_2\sqrt{x}+2\|a\|_{\infty}x)&\leq \exp(-x) \notag \\
\mathbb{P}(Z\leq -2\|a\|_2\sqrt{x})&\leq \exp(-x). \notag
\end{align}
\end{proposition}

\begin{corollary}[Corollary of Lemma 1 of \cite{laurent2000adaptive}]
\label{laurentcorollary}
Let $U$ be a $\chi^2$ statistic with $D$ degrees of freedom. For any positive $x$,
\begin{align}
\mathbb{P}(U-D\geq 2\sqrt{Dx}+2x)&\leq \exp(-x) \notag \\
\mathbb{P}(U-D\leq -2\sqrt{Dx})&\leq \exp(-x). \notag 
\end{align}
\end{corollary}

\begin{proposition}\textnormal{\textbf{(Weierstrass Product Inequality \citep[p. 210]{mitrinovic1970analytic})}}
	\label{weierstrassprodinequality}
Given real numbers $0\leq \lambda_1,\lambda_2,\dots,\lambda_n\leq 1$, the following inequality holds:
\begin{equation}
\label{weierstrass}
(1-\lambda_1)(1-\lambda_2)\cdots (1-\lambda_n) \geq 1-\sum_{k=1}^n \lambda_k.
\end{equation}
\end{proposition}

\begin{proof}\hspace{-0.07cm}\textbf{.}
	We proceed by induction. For $n=1$, the inequality is obvious. Next, assume that \eqref{weierstrass} holds for some $n$. Then,
	\begin{align}
	(1-\lambda_1)(1-\lambda_2)\cdots (1-\lambda_n)(1-\lambda_{n+1}) &\geq \left(1-\sum_{k=1}^n \lambda_k\right)(1-\lambda_{n+1}) \notag \\
	&= 1-\sum_{k=1}^n \lambda_k - \lambda_{n+1} + \lambda_{n+1}\sum_{k=1}^n \lambda_k \notag \\
	&\geq 1-\sum_{k=1}^{n+1} \lambda_k, \notag 
	\end{align}
	which shows that \eqref{weierstrass} is valid for $n+1$ as well. By induction, the statement follows.
\end{proof}

\begin{proposition} \textnormal{\textbf{(Subadditivity of the spectral radius of Hermitian matrices \cite[Fact 5.12.2]{bernstein2009matrix})}}
Let $A,B \in \mathbb{R}^{n\times n}$ be Hermitian matrices. Then,
\begin{equation}
\label{subadspectralradius}
\rho(A+B)\leq \rho(A)+\rho(B). 
\end{equation}
\end{proposition}

\begin{proof}\hspace{-0.07cm}\textbf{.}
	For this we will first prove the following chain of inequalities:
	\begin{equation}
	\label{chainineq}
	\lambda_{\textnormal{min}}(A)+\lambda_{\textnormal{min}}(B) \leq \lambda_{\textnormal{min}}(A+B) \leq \lambda_{\textnormal{max}}(A+B) \leq \lambda_{\textnormal{max}}(A)+\lambda_{\textnormal{max}}(B).
	\end{equation}
	By definition,
	\begin{equation}
	\lambda_{\textnormal{min}}(A)+\lambda_{\textnormal{min}}(B) = \min_{x\in \mathbb{R}^n\backslash\{0\}} \frac{x^\top A x}{x^\top x} + \min_{y\in \mathbb{R}^n\backslash\{0\}} \frac{y^\top B y}{y^\top y} \leq \min_{x\in \mathbb{R}^n\backslash\{0\}} \frac{x^\top (A+B) x}{x^\top x} = \lambda_{\textnormal{min}}(A+B). \notag
	\end{equation}
	Similarly,
	\begin{equation}
	\lambda_{\textnormal{max}}(A+B) = \max_{x\in \mathbb{R}^n\backslash\{0\}} \frac{x^\top (A+B) x}{x^\top x} \leq \max_{x\in \mathbb{R}^n\backslash\{0\}} \frac{x^\top A x}{x^\top x} + \max_{y\in \mathbb{R}^n\backslash\{0\}} \frac{y^\top B y}{y^\top y} =\lambda_{\textnormal{max}}(A)+\lambda_{\textnormal{max}}(B). \notag 
	\end{equation}
    With \eqref{chainineq} in mind, and using the fact that for any Hermitian matrix $H$ we can write $\rho(H)$ as $\max\{ -\lambda_{\textnormal{min}}(H),\lambda_{\max}(H)\}$ \citep[Fact 5.11.5]{bernstein2009matrix},  we see that
	\begin{equation}
	\label{comb1}
	\lambda_{\textnormal{max}}(A+B) \leq \max\{ -\lambda_{\textnormal{min}}(A),\lambda_{\max}(A)\} + \max\{ -\lambda_{\textnormal{min}}(B),\lambda_{\max}(B)\} = \rho(A)+\rho(B),
	\end{equation}
	and
	\begin{equation}
	\label{comb2}
	-\lambda_{\textnormal{min}}(A+B) \leq \max\{ -\lambda_{\textnormal{min}}(A),\lambda_{\max}(A)\} + \max\{ -\lambda_{\textnormal{min}}(B),\lambda_{\max}(B)\} = \rho(A)+\rho(B),
	\end{equation}
	Combining \eqref{comb1} and \eqref{comb2}, we reach
	\begin{equation}
	\rho(A+B) = \max\{ -\lambda_{\textnormal{min}}(A+B),\lambda_{\max}(A+B)\} \leq \rho(A)+\rho(B), \notag	
	\end{equation}
which is what we wanted to prove.
\end{proof}

\end{document}